\documentclass[manuscript, screen, dvipsnames]{acmart}
\makeatletter
\renewcommand\@formatdoi[1]{\ignorespaces}
\makeatother
\usepackage[noend]{algorithmic}
\usepackage{algorithm}
\usepackage{multirow}
\usepackage{amsmath}
\usepackage{bbm}
\usepackage[inline]{enumitem} 
\usepackage{subcaption}
\usepackage{xcolor}

\usepackage{hyperref}
\usepackage{lipsum}

\usepackage{wrapfig}

\usepackage{arydshln}
\makeatletter
\def\adl@drawiv#1#2#3{
        \hskip.5\tabcolsep
        \xleaders#3{#2.5\@tempdimb #1{1}#2.5\@tempdimb}%
                #2\z@ plus1fil minus1fil\relax
        \hskip.5\tabcolsep}
\newcommand{\cdashlinelr}[1]{%
  \noalign{\vskip\aboverulesep
          \global\let\@dashdrawstore\adl@draw
          \global\let\adl@draw\adl@drawiv}
  \cdashline{#1}
  \noalign{\global\let\adl@draw\@dashdrawstore
          \vskip\belowrulesep}}
\makeatother

\usepackage{tikz}
\usetikzlibrary{automata, arrows, bayesnet, bending}

\usepackage{tikzscale}

\copyrightyear{2023}
\acmYear{2023}
\setcopyright{none}
\acmConference[CONSEQUENCES@RecSys '23]{ACM RecSys}{September 19, 2023}{Singapore}
\acmBooktitle{CONSEQUENCES Workshop at ACM RecSys (CONSEQUENCES@RecSys'23), September 19, 2023, Singapore}

\begin{document}
\title{Offline Recommender System Evaluation under Unobserved Confounding}

\author{Olivier Jeunen}
\affiliation{
  \institution{ShareChat}
  \city{Edinburgh}
  \country{UK}
} 

\author{Ben London}
\affiliation{
  \institution{Amazon Music}
  \city{Seattle}
  \state{WA}
  \country{USA}
} 

\begin{abstract}

Off-Policy Estimation (OPE) methods allow us to learn and evaluate decision-making policies from logged data.
This makes them an attractive choice for the offline evaluation of recommender systems, and several recent works have reported successful adoption of OPE methods to this end.
An important assumption that makes this work is the absence of unobserved confounders: random variables that influence both actions and rewards at data collection time.
Because the data collection policy is typically under the practitioner's control, the unconfoundedness assumption is often left implicit, and its violations are rarely dealt with in the existing literature.

This work aims to highlight the problems that arise when performing off-policy estimation in the presence of unobserved confounders, specifically focusing on a recommendation use-case.
We focus on policy-based estimators, where the logging propensities are learned from logged data.
We characterise the statistical bias that arises due to confounding, and show how existing diagnostics are unable to uncover such cases.
Because the bias depends directly on the \emph{true} and unobserved logging propensities, it is non-identifiable.
As the unconfoundedness assumption is famously untestable, this becomes especially problematic.
This paper emphasises this common, yet often overlooked issue.
Through synthetic data, we empirically show how na\"ive propensity estimation under confounding can lead to severely biased metric estimates that are allowed to fly under the radar.
We aim to cultivate an awareness among researchers and practitioners of this important problem, and touch upon potential research directions towards mitigating its effects.
\end{abstract}

\maketitle

\section{Introduction \& Motivation}
Inferring cause and effect from observational data is not a straightforward task, due to the presence of ``\emph{unobserved confounders}''~\cite{Pearl2009}.
Indeed, if an unobserved variable exists that influences both the treatment and its outcome, this can easily lead to biased estimates of the treatment effect.
In the extreme case where the estimate changes sign, this is known as ``Simpson's Paradox''~\cite{Julious1994}.
\citeauthor{Bottou2013} describe how to avoid such unpleasantries in counterfactual learning scenarios, by carefully modelling the data-generating process~\cite{Bottou2013}.
\citeauthor{Jadidinejad2021} describe how an instance of Simpson's Paradox is prevalent to occur in the offline evaluation of recommender systems---as the test data is often influenced by an unknown recommendation policy that takes the role of the confounder~\cite{Jadidinejad2021}.
Their proposed solution to this problem involves \emph{Inverse Propensity Score} (IPS) weighting~\cite[Ch. 9]{Owen2013}, with propensities that are estimated from logged data.
Estimating propensities is common practice in cases where the true propensities are unknown~\cite{Yang2018,Ai2018}, and empirical results indicate that IPS with estimated propensities can even lead to favourable variance~\cite{Hanna2019}.
Whether to simplify experimental setup or to deal with missing information in publicly available datasets, these works often make the limiting assumption that the propensities are independent of user or context (see \cite[\S 6.3]{Jadidinejad2021} or \cite[\S 3.3]{Yang2018}).
In the very likely case that these assumptions are violated, this implies the presence of unobserved confounders.%
\footnote{An analogous problem occurs in ranking applications when contextual independence is incorrectly assumed for position bias estimates~\cite{Fang2019,Jeunen2023_C3PO}.}
Existing diagnostics for validating logged bandit feedback cannot detect these issues~\cite{Li2015,London2022}, as we will show formally. 

Unobserved confounding is a well-known issue in general off-policy reinforcement learning, and several methods have been proposed to deal with it in the recent literature.
They typically leverage additional data (such as interventions~\cite{Wang2021,Gasse2021} or instrumental variables~\cite{Xu2023}) or make assumptions about the nature of confounding variables~\cite{Namkoong2020,Bennett2021,Kausik2023}.
Analogously, instrumental variables have been leveraged to test for the unconfoundedness assumption~\cite{DeLuna2014}, and other statistical methods have been proposed to assess the sensitivity of results to potential confounders~\cite{Liu2013}.
Nevertheless, in the absence of additional tools, unconfoundedness is a famously \emph{untestable} assumption, rendering its effects especially troublesome.
Focusing on the off-policy bandit setting with a guiding example in recommendation, we aim to answer:
\begin{center}
    ``\emph{Can we reliably select the optimal policy from a set of competing policies, under unobserved confounding?}''
\end{center}
\section{Off-Policy Estimation in the Presence of Unobserved Confounders}
\begin{wrapfigure}{rt}{0.25\textwidth}
\centering
\vspace{-1ex}
\includegraphics{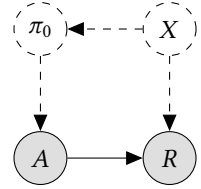}
    \caption{Probabilistic Graphical Model (PGM) for our setup.}\vspace{-2ex}
    \label{fig:PGM}
\Description[Figure described in text.]{Figure described in text.}
\end{wrapfigure}

Throughout this work, we denote random variables as $X$, with specific instances as $x \in \mathcal{X}$.
A contextual bandit problem consists of \emph{contexts} $X$ (e.g. user and context features), \emph{actions} $A$ (e.g. item recommendations), and \emph{rewards} $R$ (e.g. clicks, streams, revenue).
Rewards are causally influenced by both contexts and actions, as illustrated by the edges in the causal graph shown in Figure~\ref{fig:PGM}.
A contextual \emph{policy} $\pi$ determines which actions are selected (or sampled), thereby inducing a probability distribution over $A$, which is often denoted with the shorthand $\pi(a|x) \coloneqq \mathsf{P}(A=a|X=x; \Pi=\pi)$.
A policy's effectiveness is measured by the expected reward obtained when selecting actions according to that policy: $\mathbb{E}_{x} \mathbb{E}_{a \sim \pi(\cdot|x)}[R]$.
In a recommendation application, this value can be estimated by deploying the policy in an online experiment.
However, since such experiments are typically costly, and we may have many policies to evaluate, we would rather obtain reward estimates by other means.

Suppose there is an existing deployed policy, called the \emph{logging policy} $\pi_{0}$, with which we collect a dataset $\mathcal{D}_{0} \coloneqq \{(a_{i}, r_{i})_{i=1}^{N}\}$.
We will assume, as is often the case, that the logging policy, and the contextual covariates it uses, are unobservable, as indicated by the dashed nodes and edges in Figure~\ref{fig:PGM}.
Our goal is to leverage data logged under $\pi_{0}$ to estimate the expected reward under $\pi$---a problem often referred to as \emph{off-policy} estimation~\cite{Vasile2020,Saito2021}.

One simple off-policy estimator is the \emph{Direct Method} (DM).
DM computes the expected reward under $\pi$ (Eq.~\ref{eq:DM_policy}) using a model $\widehat{R}^A_{\rm DM}(a)$ that estimates the reward for every available action.
Since we assume that contextual covariates are unavailable, the best we can do for $\widehat{R}^A_{\rm DM}(a)$ is to na\"ively count the observed rewards for every action (Eq.~\ref{eq:DM_action}).
\begin{minipage}{0.4\textwidth}
\begin{equation}\label{eq:DM_policy}
    \widehat{R}_{\rm DM}(\pi) =  \sum_{a \in \mathcal{A}} \widehat{R}^A_{\rm DM}(a) \pi(a)
\end{equation}
\end{minipage}
\hspace{10ex}
\begin{minipage}{0.45\textwidth}
\begin{equation}\label{eq:DM_action}
    \widehat{R}^A_{\rm DM}(a) = \frac{\sum_{(a_{i},r_{i}) \in \mathcal{D}_{0}} \mathbbm{1}\{a_{i} = a\} \cdot r_{i}}{\sum_{(a_{i},r_{i}) \in \mathcal{D}_{0}} \mathbbm{1}\{a_{i} = a\}}
\end{equation}
\end{minipage}

\noindent
Unfortunately, this estimator is biased, for two reasons:
\begin{enumerate*}
    \item[(a)] it does not take into account the covariates $X$ (i.e., the model is mis-specified), and
    \item[(b)] it ignores the selection bias from the logging policy $\pi_{0}$, influencing the estimates in Eq.~\ref{eq:DM_action}.
\end{enumerate*}

In theory, we can bypass both the model mis-specification and selection bias problems by leveraging the \emph{ideal} IPS estimator (Eq.~\ref{eq:ideal_IPS_policy}), which is provably unbiased.
Importantly, ideal IPS requires access to both the contextual covariates and the exact action probabilities (\emph{propensities}) under $\pi_{0}$---which we assume are unavailable.
Accordingly, we will adopt the common practice of using estimated logging propensities $\widehat{\pi}_{0}$ for IPS (Eq.~\ref{eq:estim_IPS_policy}).
As the estimated propensities cannot properly consider all covariates, this leads to unobserved confounding.

\begin{minipage}{0.45\textwidth}
\begin{equation}\label{eq:ideal_IPS_policy} 
    \widehat{R}_{\rm ideal-IPS}(\pi) = \frac{1}{\left|\mathcal{D}_{0}\right|}\sum_{(\textcolor{Maroon}{x_{i}}, a_{i},r_{i}) \in \mathcal{D}_{0}} r_{i}\frac{\pi(a_{i})}{ \textcolor{Maroon}{\pi_{0}(a_{i}|x_{i})}}
\end{equation}
\end{minipage}
\hspace{5ex}
\begin{minipage}{0.45\textwidth}
\begin{equation}\label{eq:estim_IPS_policy} 
    \widehat{R}_{\rm estim-IPS}(\pi) = \frac{1}{\left|\mathcal{D}_{0}\right|}\sum_{(a_{i},r_{i}) \in \mathcal{D}_{0}} r_{i}\frac{\pi(a_{i})}{ \widehat{\pi}_{0}(a_{i})}
\end{equation}
\end{minipage}

Using the fact that the ideal IPS estimator is unbiased, we can quantify the bias of the estimated IPS estimator as:
\begin{equation}
       \mathbb{E}[\widehat{R}_{\rm estim-IPS}(\pi)] - \mathbb{E}_{a \sim \pi}[R]
   =
       \mathbb{E}[\widehat{R}_{\rm estim-IPS}(\pi)] - \mathbb{E}[\widehat{R}_{\rm ideal-IPS}(\pi)]
   =
       \mathbb{E}\left[R\pi(A|X) \left(\frac{1}{\widehat{\pi}_{0}(A)}- \frac{1}{\pi_{0}(A|X)} \right)\right]
   .
\end{equation}

To further illustrate our point, we resort to \citeauthor{Pearl2009}'s do-calculus framework~\cite{Pearl2009}.
What OPE methods wish to estimate is the expected value of the reward given that a new policy \emph{intervenes} on the action distribution.
When unobserved confounders are present, this \emph{interventional} quantity is \emph{not} equal to the \emph{observational} quantity we can estimate from logged data: $\mathop{\mathbb{E}}\left[R|A=a\right] \neq \mathop{\mathbb{E}}\left[R|{\rm do}(A=a)\right] $.
Instead, we would require the ``backdoor adjustment'' to obtain:
\begin{equation}
 \mathop{\mathbb{E}}\left[R|{\rm do}(A=a)\right] = \sum_{\textcolor{Maroon}{x \in \mathcal{X}}} \mathop{\mathbb{E}}\left[  R | A=a, \textcolor{Maroon}{X=x} \right] .
\end{equation}
It should be clear that without access to $X$, this estimand is non-identifiable, and this problem is not easily solved.

\section{Existing Diagnostics for Logging Propensities do not Uncover Confounding Bias}
Several diagnostics have been proposed in the literature to detect data quality issues with logged bandit feedback.
In particular, they try to uncover cases where the two classical assumptions of the IPS estimator do not hold~\cite{Li2015,London2022}: 
\begin{enumerate*}
    \item either the empirical action frequencies in the data do not match those implied by the logged propensities, or
    \item the logging policy does not have full support over the action space.
\end{enumerate*}
Note that the presence of unobserved confounders does \emph{not} automatically violate these assumptions.
As a result, the diagnostics that were proposed will \emph{not} detect confounding bias.
Logging propensities can be estimated by empirically counting logged actions, as shown in Eq.~\ref{eq:marginal_prop}.
In doing so, we obtain unbiased estimates of the true marginal action probabilities.
Indeed, $\lim_{N\to\infty} \widehat{\pi_{0}}(a)= \mathsf{P}(A=a|\Pi=\pi_0)$.

\citeauthor{Li2015} propose the use of \emph{arithmetic} and \emph{harmonic} mean tests to compare empirical action frequencies with the logging propensities~\cite{Li2015}.
As we \emph{define} the logging propensities to be equal to the empirical action frequencies, it should be clear that this test will trivially pass.
Alternatively, \citeauthor{London2022} propose to use the average importance weight as a control variate, whose expectation should equal 1 for any target policy $\pi$~\cite{London2022}.
Here as well, because the marginal propensities are unbiased (Eq.~\ref{eq:marginal_prop}), we can show that the control variate remains unbiased as well (Eq.~\ref{eq:CV_marginal}).
\begin{equation}\label{eq:marginal_prop}
    \widehat{\pi}_{0}(a)
    = \frac{1}{|\mathcal{D}_{0}|}\sum_{(a_{i},r_{i}) \in \mathcal{D}_{0}} \!\!\!\! \mathbbm{1}\{a_{i} = a\} \underset{\lim\limits_{N\to \infty}}{=} \mathsf{P}(A=a|\Pi=\pi_0) = \sum_{x \in \mathcal{X}}  \mathsf{P}(A=a|X=x, \Pi=\pi_0) \mathsf{P}(X=x).
\end{equation}

\begin{theorem}
When unobserved confounders are present and logging propensities are estimated from empirical data (thus ignoring the confounders), the expected value of the importance weights equals 1 for any target policy:
$$\mathop{\mathbb{E}}_{\substack{x \sim \mathsf{P}(X)\\ a \sim \mathsf{P}(A|X=x,\Pi=\pi_{0})}}\left[ \frac{\pi(a)}{\widehat{\pi}_0(a)} \right] = 1.$$
\end{theorem}
\begin{proof}
    
\begin{equation}\label{eq:CV_marginal}
\begin{split}    
\mathop{\mathbb{E}}_{\substack{x \sim \mathsf{P}(X)\\ a \sim \mathsf{P}(A|X=x,\Pi=\pi_{0})}}\left[ \frac{\pi(a)}{\widehat{\pi}_0(a)} \right]
&=  \sum_{a \in \mathcal{A}} \sum_{x \in \mathcal{X}}\frac{\pi(a)}{\widehat{\pi}_0(a)} \mathsf{P}(A=a|X=x,\Pi=\pi_0)\mathsf{P}(X=x)\\
&=  \sum_{a \in \mathcal{A}} \frac{\pi(a)}{\widehat{\pi}_0(a)} \sum_{x \in \mathcal{X}} \mathsf{P}(A=a|X=x,\Pi=\pi_0)\mathsf{P}(X=x)\\
&\underset{\lim\limits_{N\to \infty}}{=}  \sum_{a \in \mathcal{A}} \pi(a)\frac{\sum_{x \in \mathcal{X}} \mathsf{P}(A=a|X=x,\Pi=\pi_0)\mathsf{P}(X=x)}{\sum_{x \in \mathcal{X}} \mathsf{P}(A=a|X=x,\Pi=\pi_0)\mathsf{P}(X=x)} = \sum_{a \in \mathcal{A}} \pi(a) = 1 \qed\\
\end{split}
\end{equation}
\end{proof}

As such, existing diagnostics are unable to detect issues of unobserved confounding.
This implies that the self-normalised IPS (SNIPS) estimator and its extensions that adopt the above control variate to reduce the variance of the IPS estimator, would exhibit the same bias as estimated IPS when unobserved confounders are present~\cite{Swaminathan2015,Joachims2018}.

\section{Empirical Validation of the Effects of Unobserved Confounding on Synthetic Data}
We now describe a guiding example, and provide a notebook that implements the methods described earlier at \href{https://github.com/olivierjeunen/confounding-consequences-2023/}{github.com/olivierjeunen/confounding-consequences-2023/}.
Consider a setting with two possible actions and a binary covariate $X = \{x_0, x_1\}$, following the distribution in Table~\ref{tab:example_cov} (parameterised with $\alpha \in \left[\frac{1}{2},1\right]$).
Rewards are Bernoulli-distributed (Table~\ref{tab:example_reward}).
The logging policy is contextual, taking a suboptimal action with probability $\epsilon \in [0,1]$ (Table~\ref{tab:example_pi0}).
We can map this to an intuitive setting: action $a_{1}$ is of general appeal to the entire population (i.e. $R \perp \!\!\! \perp X | A=a_{1}$); whereas action $a_{0}$, on the other hand, is specifically appealing to a more niche user-base (i.e. $\mathbb{E}[R|X=x_0,A=a_0]>\mathbb{E}[R|X=x_0,A=a_1]$, but $\mathsf{P}(X=x_{0})<\mathsf{P}(X=x_{1})$).
Estimates for logging propensities can be obtained by empirical counting, as in Eq.~\ref{eq:marginal_prop}.
The expected value for these estimated context-independent propensities is shown in Table~\ref{tab:example_pihat}.

\textit{Na\"ive propensity estimation methods suffer from confounding bias.}
We simulate an off-policy estimation setup where we wish to evaluate deterministic policies $\pi_{a}(a)\equiv 1$.
We obtain $N=2\cdot 10^{6}$ samples from the synthetic distribution described in Table~\ref{tab:example}, and compute the confounded estimate $\widehat{R}_{\rm estim-IPS}$, as well as the unobservable ideal IPS estimate $\widehat{R}_{\rm ideal-IPS}$.
We vary both the level of selection bias $\epsilon$ (over the x-axis), and the confounding distribution $\alpha$ (over columns) in Fig.~\ref{fig:results}, where the y-axis shows the estimated difference in rewards from policies $\pi_{a_{1}}$ and $\pi_{a_{0}}$.
We shade the positive and negative regions in the plot to clearly visualise when an off-policy estimator allows us to \emph{correctly} identify the optimal policy, or when it does not.
We observe that the IPS estimator with estimated propensities fails considerably, in that it will incorrectly identify $\pi_{a_{0}}$ as the reward-maximising policy.
Only when $\epsilon$ is sufficiently high (i.e. approaching a uniform logging policy for $\epsilon=0.5$, and hence no confounding is present), $\widehat{R}_{\rm estim-IPS}$ is able to correctly identify $\pi_{a_{1}}$.
This shows that, even in simplified settings, the estimates we obtain from IPS with confounded estimates lead to misleading conclusions.
Furthermore, existing diagnostics cannot detect these problems when they occur.

\begin{table}[t]
\vspace{-2ex}
\begin{subtable}[t]{0.22\textwidth}
    \centering
    \begin{tabular}{cc}
    \toprule
      $\mathsf{P}(X=x_0)$ &  $\mathsf{P}(X=x_1)$ \\
    \midrule
        $1-\alpha$   &  $\alpha$ \\
    \bottomrule
     ~ & ~\\
    \end{tabular}\caption{Covariate distribution}\label{tab:example_cov}
\end{subtable}
\hfill
 \begin{subtable}[t]{0.22\textwidth}
     \centering
     \begin{tabular}{rcc}
     \toprule
    $\mathbf{\mathbb{E}[R|X,A]}$& $a_0$ &  $a_1$ \\
     \midrule
        $x_0$   &  $1.0$ & $0.7$ \\
        $x_1$   &  $0.0$ & $0.7$ \\
     \bottomrule
     \end{tabular}\caption{Reward distribution}\label{tab:example_reward}
 \end{subtable}
 \hfill
 \begin{subtable}[t]{0.25\textwidth}
     \centering
     \begin{tabular}{rcc}
     \toprule
      $\mathbf{\pi_{0}(a|x)}$ & $a_0$ &  $a_1$ \\
     \midrule
        $x_0$   &  $1.0-\epsilon$ & $\epsilon$ \\
        $x_1$   &  $\epsilon$ & $1.0-\epsilon$ \\
     \bottomrule
     \end{tabular}\caption{Logging policy distribution}\label{tab:example_pi0}
 \end{subtable}
  \hfill
 \begin{subtable}[t]{0.25\textwidth}
     \centering
     \begin{tabular}{cc}
     \toprule
      \multicolumn{2}{c}{$\mathbf{\widehat{\pi}_{0}(a)}$}\\
     \midrule
        $a_0$   &  $(1-\epsilon)\alpha + \epsilon (1-\alpha)$\\
        $a_1$   &  $(1-\epsilon)(1-\alpha) + \epsilon \alpha$\\
     \bottomrule
     \end{tabular}\caption{Logging propensity estimates}\label{tab:example_pihat}
 \end{subtable}

\caption{Data distributions for a guiding example that highlights issues with off-policy estimation under unobserved confounding.} \label{tab:example}
\end{table}

\begin{figure}
    \vspace{-5ex}
    \centering
    \includegraphics[width=\linewidth]{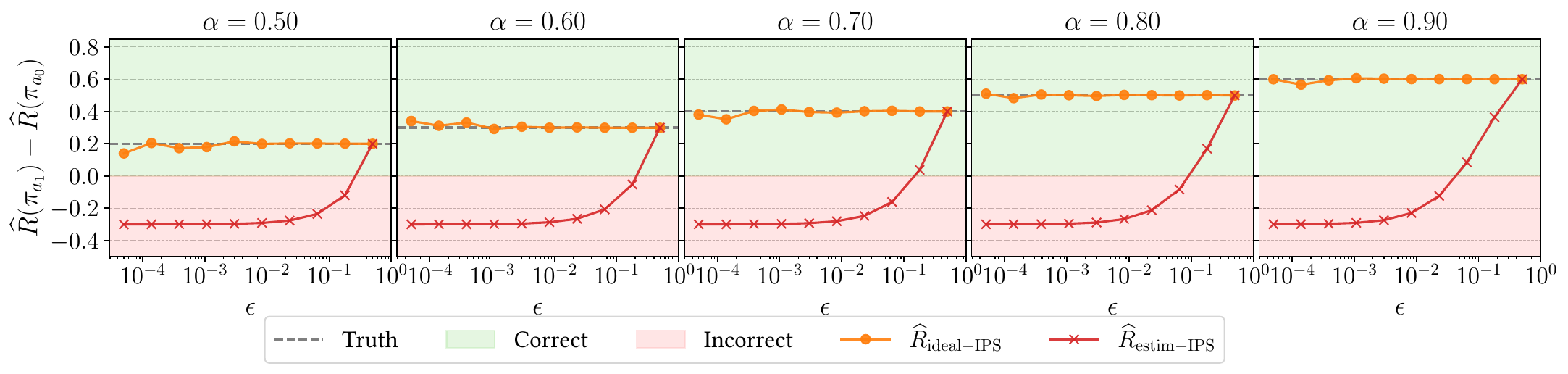}
    \caption{Estimated differences in reward for various estimators. IPS with estimated propensities suffers from unobserved confounding.}
    \label{fig:results}
\end{figure}

\section{Conclusions \& Outlook}
Unobserved confounders lead to biased estimates, both for DM- and IPS-based methods.
This problem has received considerable attention in the research literature for general offline reinforcement learning use-cases, but the literature dealing with these issues in recommendation settings remains scarce.
Our work highlights that this is problematic---especially in cases where propensities are estimated under simplifying independence assumptions.
In doing so, we add to the literature identifying problematic practices that might hamper progress in the field~\cite{Jeunen2019,Krichene2020, Sun2023,Jeunen2023_nDCG}.

\bibliographystyle{ACM-Reference-Format}
\bibliography{bibliography}

\end{document}